\documentclass[letterpaper, 10 pt, conference]{ieeeconf}

\IEEEoverridecommandlockouts  
\usepackage[margin=1in]{geometry}

\usepackage[utf8]{inputenc}
\usepackage{cite}
\usepackage[dvipsnames]{xcolor}
\usepackage{hyperref}
\hypersetup{    
    colorlinks=true,        
    linkcolor=NavyBlue,
    urlcolor=NavyBlue,
    citecolor=Mahogany     
}
\usepackage[normalem]{ulem}
\usepackage{graphicx} 
\usepackage{amsmath} 
\usepackage{mathtools}
\usepackage{amssymb}  
\usepackage{mathrsfs}

\usepackage{amsthm}
\usepackage[ruled,linesnumbered]{algorithm2e}
\usepackage{algpseudocode}
\usepackage{circledsteps}

\usepackage{soul}
\usepackage{caption}
\usepackage{subcaption}

\usepackage{multirow}
\usepackage{multicol}
\usepackage{booktabs}
\usepackage{cellspace}  
\usepackage[bb=boondox]{mathalfa}
\newcommand{\bbzero}{\mathbb{0}}

\newcommand{\bbb}{\mathbb{B}}

\newcommand{\bbe}{\mathbb{E}}
\newcommand{\bbm}{\mathbb{M}}

\newcommand{\bbp}{\mathbb{P}}
\newcommand{\bbq}{\mathbb{Q}}
\newcommand{\bbr}{\mathbb{R}}

\newcommand{\bbz}{\mathbb{Z}}

\newcommand{\cala}{\mathcal{A}}

\newcommand{\calc}{\mathcal{C}}

\newcommand{\calh}{\mathcal{H}}

\newcommand{\calm}{\mathcal{M}}
\newcommand{\caln}{\mathcal{N}}
\newcommand{\calo}{\mathcal{O}}
\newcommand{\calp}{\mathcal{P}}
\newcommand{\calr}{\mathcal{R}}

\newcommand{\calt}{\mathcal{T}}
\newcommand{\calu}{\mathcal{U}}

\newcommand{\caly}{\mathcal{Y}}

\usepackage{bm}  
\newcommand*\rv[1]{\bm{#1}}  

\newcommand{\numsamp}{N_s} 

\DeclareMathOperator{\cvar}{CVaR}
\DeclareMathOperator{\drcvar}{DR-CVaR}

\DeclareMathOperator{\wce}{WCE}
\DeclareMathOperator{\dwass}{d_w}

\newcommand{\Nob}{N_{ob}}

\newcommand{\horizon}{T}

\newcommand{\supportFxn}[2]{S_{#1}(#2)}

\newcommand{\norm}[1]{\left\lVert {#1} \right\rVert}

\usepackage{xcolor}

\newtheorem{rem}{Remark}
\newtheorem{lemma}{Lemma}

\newtheorem{prob}{Problem}
\newtheorem{prop}{Proposition}
\newtheorem{defn}{Definition}
\newtheorem{example}{Example}

\title{\LARGE \bf 
Distributionally Robust CVaR-Based Safety Filtering for Motion Planning in Uncertain Environments}

\author{\large Sleiman Safaoui and Tyler H. Summers
\thanks{
This work was supported in part by the United States Air Force Office of Scientific Research under Grants FA2386-19-1-4073 and FA9550-23-1-0424, and in part by the National Science Foundation under Grant ECCS-2047040.}
\thanks{
S. Safaoui and T. H. Summers are with the Erik Jonsson School of Engineering and Computer Science, The University of Texas at Dallas, Richardson, TX, USA. E-mail: \{sleiman.safaoui, tyler.summers\}\@utdallas.edu.
}
}

\begin{document}

\maketitle

\begin{abstract}
Safety is a core challenge of autonomous robot motion planning, especially in the presence of dynamic and uncertain obstacles. 
Many recent results use learning and deep learning-based motion planners and prediction modules to predict multiple possible obstacle trajectories and generate obstacle-aware ego robot plans. 
However, planners that ignore the inherent uncertainties in such predictions incur collision risks and lack formal safety guarantees.
In this paper, we present a computationally efficient safety filtering solution to reduce the collision risk of ego robot motion plans using multiple samples of obstacle trajectory predictions. 
The proposed approach reformulates the collision avoidance problem by computing safe halfspaces based on obstacle sample trajectories using distributionally robust optimization (DRO) techniques. 
The safe halfspaces are used in a model predictive control (MPC)-like safety filter to apply corrections to the reference ego trajectory thereby promoting safer planning. 
The efficacy and computational efficiency of our approach are demonstrated through numerical simulations.
\end{abstract}

\section{Introduction}
Autonomous robots have many application areas including 
autonomous driving \cite{yurtsever2020survey},   
warehouse management and logistics \cite{wawrla2019applications},  
drone delivery \cite{jung2017analysis}, 
and agriculture \cite{das2015devices}.  
A core challenge facing autonomous robots is navigation in dynamic and uncertain environments, i.e. in the presence of moving obstacles whose future motion cannot be predicted exactly. 
This scenario complicates the robot safety requirements: the ego robot must presume the dynamic obstacles' intentions and predict their future trajectories for use in computing its own motion plan. 
Thus, safety hinges on how accurately the dynamic obstacles' behavior can be predicted \cite{martinez2017driving, yurtsever2019risky}. 
Failing to account for prediction uncertainties may incur undue risk of severe collisions \cite{davies2016google}.

Various methods have been studied for predicting how obstacles will behave, but it is still an active area of research.
In \cite{geng2017scenario}, a hidden Markov model is used for better understanding urban scenarios for autonomous vehicles (AVs).
In another work, \cite{kumar2013learning} uses a support vector machine and Bayesian filtering to predict lane change intentions for AVs.
Furthermore, deep learning approaches have also been used. 
End-to-end motion planners, such as \cite{philion2020lift}, implicitly account for future predictions, but they fail to explicitly capture the environment uncertainties which may lead to collisions.
FIERY \cite{hu2021fiery} generates a birds-eye-view probabilistic future predictions map which estimates environmental uncertainties, but it still requires a formal approach to use this data to enforce safety.
In \cite{zhou2022long}, a neural network ensemble is employed to estimate prediction uncertainty and identify rare cases. However, ensembles are resource-intensive to train and deploy.

\begin{figure}
    \centering
    \includegraphics[width=\linewidth, trim={2.2cm 3.2cm 2.2cm 2.7cm}, clip]{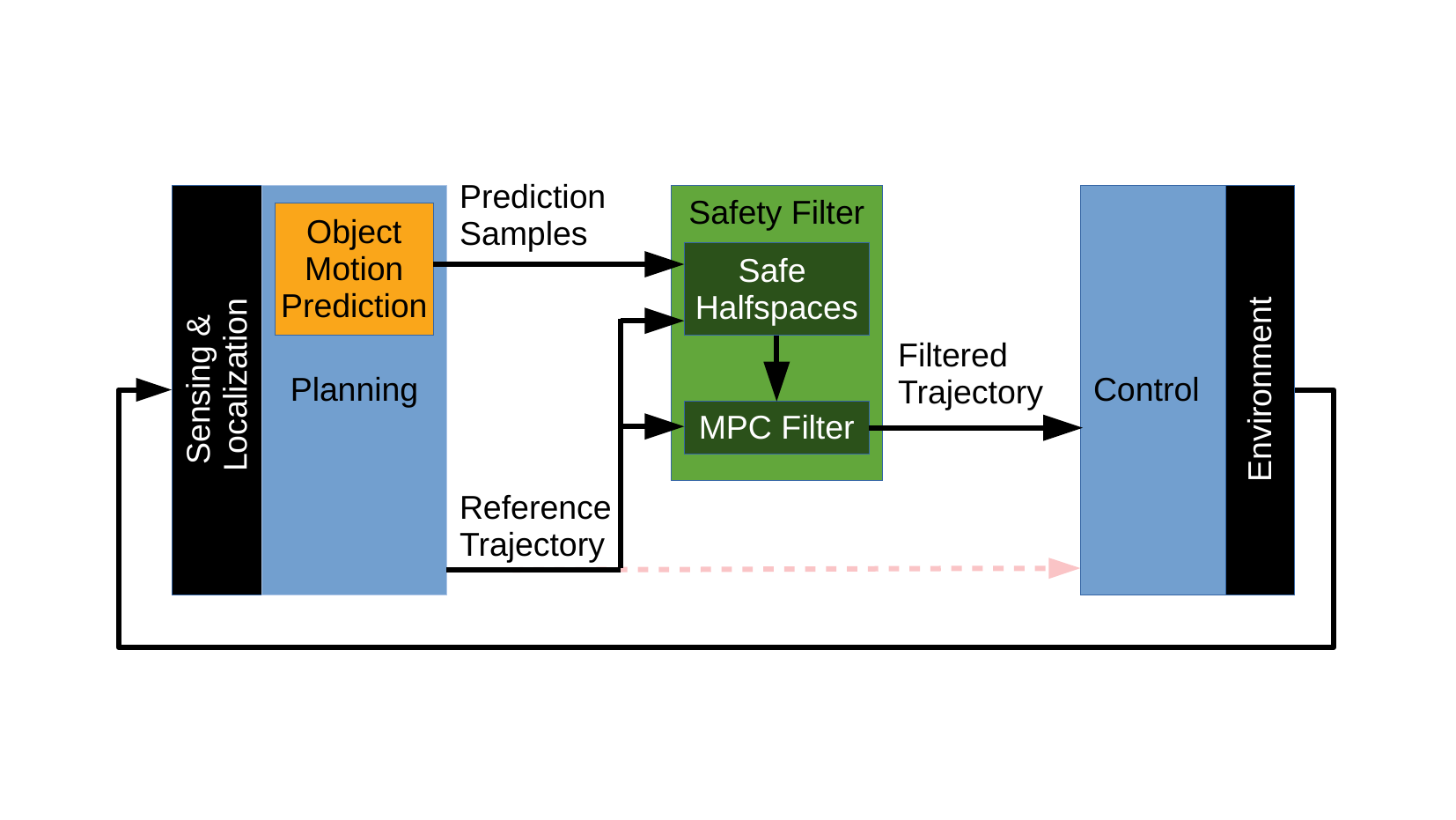}
    \caption{An autonomy stack with the proposed \emph{Safety Filter} module. The module intercepts the reference trajectory and corrects it to enforce the safety requirement.}
     \label{fig:autonomy_stack}
    \vspace{-0.3cm}
\end{figure}

Optimization-based methods are commonly used to formally guarantee safety requirements.
For obstacle motion with additive noise, \cite{hakobyan2021wasserstein} proposes a solution that uses noise samples to create an empirical distribution. It then formulates a \emph{distributionally robust optimization} (DRO) problem to ensure safety and avoid collisions under any distribution that is ``close'' to the empirical one. However, this method solves a non-convex problem which is computationally demanding and not suitable for real-time operation.
Another recent solution uses conformal prediction to guarantee safety when using learning-based planners \cite{lindemann2022safe}. Here, prediction regions that satisfy a given probability bound are found and used in an MPC optimization problem.

A related approach to enforcing safety uses a \emph{safety filter}.  
Instead of adding constraints to one of the modules of the autonomy stack, a standalone module takes the reference trajectory from the motion planner and outputs a corrected \emph{filtered} trajectory, as illustrated in Figure \ref{fig:autonomy_stack}. The filtered motion plan is guaranteed to satisfy certain safety requirements. 
Safety filters have been used in \emph{deterministic} settings for autonomous racing \cite{tearle2021predictive}, multi-agent motion planning \cite{vinod2022safe}, and autonomous driving \cite{phan2022driving} to filter unsafe learning-based motion plans.

Motion planners with prediction uncertainty often suffer from two issues. 
1) They consider average behavior or limit the chance of unsafe events. These are suitable for real-time deployment but fail under edge cases (in the prediction distribution tail).
2) They rigorously tackle uncertainties and edge cases, but they are computationally intensive.
In this work, we address this gap by proposing a computationally efficient solution using axiomatic risk theory to handle uncertainties and deal with edge cases.
We assume that the ego robot has a planned reference trajectory and samples of the obstacles' future trajectories (e.g. through \cite{salzmann2020trajectron++}).
Our solution starts by finding safe halfspaces based on a distributionally robust conditional value-at-risk ($\drcvar$) risk metric for each obstacle. 
Then, the $\drcvar$ safe halfspaces are used in an MPC-based safety filter to enforce safety.
The main contributions of this work are:
\begin{itemize}
    \item We extend the notion of a safe halfspace and define data-driven $\drcvar$ safe halfspaces that bound the risk of violating a safety specification.
    \item We verify the efficiency of the $\drcvar$ halfspaces through a numerical analysis and show that they can be computed in a few milliseconds with up to a few hundred samples.
    \item We formulate an MPC-based safety filter that uses the $\drcvar$ safe halfspaces to constrain the motion planning problem and bound collision risks.
    \item We demonstrate the efficacy of our solution and its ability to handle edge-cases through numerical simulations in a variety of motion planning scenarios.
\end{itemize}

\section*{Notation}
The $d$-dimensional zero vector (matrix) and identity matrix are denoted by $\bbzero_{d}$ ($\bbzero_{d,d}$) and $I_d$, respectively. 
We use $a:b$ to denote all integers between $a \in \bbz$ and $b \in \bbz$ (inclusive). 
The Minkowski sum is denoted by $\oplus$. 
The transpose of a vector or matrix is denoted by $(\cdot)^\top$. 
The inner product of vectors $z_1$ and $z_2$ is denoted by $z_1 \cdot z_2 = z_1^\top z_2$. 
The support function of a compact set $\calc$ is given by $\supportFxn{\calc}{z} := \sup_{x\in \calc} z \cdot x$. 
Random variables/vectors are denoted in $\rv{bold}$ and 
$\bbe[\cdot]$ is the expected value operator.
Given a loss $\rv{l}$, the $\cvar$ metric is $\cvar_\alpha^\bbp(\rv{l}):= \inf_{\tau \in \bbr} \bbe^{\bbp} [\tau + \frac{1}{\alpha}\max \{\rv{l}-\tau, 0\}]$ which is evaluated with respect to the $\alpha$ worst-cases of the distribution $\bbp$ (the $1-\alpha$ quantile in the upper tail).
For $\numsamp$ samples of the loss $\{l^1, \dots, l^{\numsamp}\}$, we use a sample average approximation to evaluate the expected value in the $\cvar$ metric:\\ 
$\cvar_\alpha^\bbp(\rv{l}) \approx \inf_{\tau \in \bbr} \frac{1}{\numsamp} \sum_{i=1}^{\numsamp} (\tau + \frac{1}{\alpha}\max \{l^i-\tau, 0\}).$

\section{DR-CVaR Safe Motion Planning with Uncertain Dynamic Obstacles}
\subsection{Motion Planning Safety Filtering Problem}
We consider the problem of motion planning for an ego robot in the presence of $\Nob$ dynamic obstacles whose future behavior is uncertain.
The ego robot dynamics are assumed linear and described by:
\begin{subequations} \label{eq:ego_full_dyn}
    \begin{align}
        x_{t+1} &= Ax_t + Bu_t \label{eq:ego_dyn}\\
        y_t &= Cx_t \label{eq:ego_observe}
    \end{align}
\end{subequations}
where 
$x_t \in \bbr^n$ is the robot state at time $t$, 
$u_t \in \bbr^m$ is the control, 
$A \in \bbr^{n \times n}, \ B \in \bbr^{n \times m}$ are the state and input matrices, 
$y \in \bbr^d$ is the position, 
and 
$C \in \bbr^{d \times n}$ is the matrix that extracts the position from the state.
For simplicity, we consider 2D position ($d=2$), but the results can be easily extended to 3D.
The ego robot geometry is assumed to be a convex and compact set denoted by $\cala$.
We also assume that the ego robot has a desired reference trajectory over a horizon of length $\horizon$ given by $\calt^r(t) := \{x^r(t), \dots, x^r(t+\horizon)\}$.
This reference trajectory may be obtained from any planning module 
(e.g. neural network, sampling-based or optimization-based methods, etc.).

Each obstacle is modeled as a convex and compact set $\calo_i \ \forall i \in [1:\Nob]$. 
Their dynamics are unknown and the motion is uncertain. Furthermore, the prediction distribution is inherently unknown.  
Instead, we assume access to a module that can generate \emph{sample predictions} of the obstacles' trajectories. 
The $s$-th sample trajectory of the $i$-th obstacle for a horizon of $\horizon$ time steps is given by $\calt_i^s(t) := \{p_i^s(t), \dots, p_i^s(t+\horizon)\}$
and includes only position information ($p_i^s \in \bbr^d$).

An autonomy stack may pass the reference trajectory directly to a tracking controller (Figure~\ref{fig:autonomy_stack} faded red arrow) while assuming that it accounts for the obstacles' future trajectories.
However, safety requirements desired in a reference trajectory may not be guaranteed, especially if the planner is based on a neural network. A formal safety guarantee is defined as follows.

\begin{defn}[Safety Guarantee]\label{def:safety}
    Given a reference trajectory $\calt^r$, a chosen risk metric $\calr$, and a risk bound $\delta$, safety is guaranteed iff $\calr(\calt^r) \leq \delta$. 
\end{defn}

To enforce this notion of safety, we advocate for the usage of a safety filter that takes a reference trajectory and outputs a filtered trajectory as depicted in Figure \ref{fig:autonomy_stack}. This safety filtering problem is formalized below.
\begin{prob}[Motion Planning Safety Filtering] \label{pb:prob_statement}
    Given an ego reference trajectory $\calt^r(t)$ and obstacle trajectory samples $\calt_i^s(t), \ s \in [1:\numsamp]$ for every obstacle $i \in [1:\Nob]$, find a filtered ego trajectory that ensures the safety of the ego robot per Definition \ref{def:safety}.
\end{prob}
We now proceed with specifying the adopted risk metric $\calr$ for the motion planning problem.

\subsection{Signed Collision Loss Function}
Consider an ego robot and a single dynamic obstacle. 
The position of the ego robot is denoted by $y$, the ego reference position is denoted by $y^r$, and the obstacle position is denoted by $p$.

The deterministic collision avoidance condition is given by $(y \oplus \cala) \cap (p \oplus \calo) = \emptyset$. 
Using computational geometry arguments and convexifying the constraint via a separating hyperplane, as in \cite[III-A.2]{vinod2022safe}, we have:
\begin{align}
    &(y \oplus \cala) \cap (p \oplus \calo) = \emptyset 
    \iff  (y-p) \not\in  \calo \oplus (-\cala) \nonumber \\
    &\Longleftarrow z \cdot (y-p) \geq S_{\calo}(z) + S_{-\cala}(z) \label{eq:halfspace_constr_det}
\end{align}
where $z$ is any chosen unit vector per \cite{boyd2004convex}.

The constraint \eqref{eq:halfspace_constr_det} is sufficient for ensuring collision avoidance for a deterministic system.
Additionally, we can define a deterministic \emph{safe halfspace} $$\calh := \{y \mid h \cdot y + g - (S_{\calo}(h) + S_{-\cala}(h)) \leq 0\}$$ such that if $y \in \calh$, then $y$ is collision free and satisfies~\eqref{eq:halfspace_constr_det}. 
For the deterministic case, the parameters $h,\ g$ of $\calh$ can be directly mapped to the values in \eqref{eq:halfspace_constr_det}. 
Furthermore, using $\calh$ we can define a signed distance function that quantifies the violation or satisfaction amount of a point relative to the collision avoidance constraint. 
In particular, consider the following loss function:
\begin{align}
    \ell(p) = -(h\cdot p + \underbrace{g - (S_{\calo}(h) + S_{-\cala}(h))}_{\Tilde{g}:=}). \label{eq:loss_fxn_with_support}
\end{align}
For an obstacle position $p$, if $\ell(p) \geq 0$ then $p$ intrudes into the safe halfspace $\calh$ and $\ell(p)$ is the intrusion amount. Otherwise, the obstacle is $|\ell(p)|$ units away from the boundary of the safe halfspace.

\subsection{Collision Avoidance using DR-CVaR Safe Halfspaces}
When the obstacle position is a random variable $\rv{p}$, \eqref{eq:halfspace_constr_det} is ill-posed and the loss \eqref{eq:loss_fxn_with_support} becomes a random variable $\ell(\rv{p})$.
We now define a risk-based safe halfspace with respect to a risk metric $\calr$ applied to $\ell(\rv{p})$.

\begin{defn}[Risk-Based Safe Halfspace] \label{def:risk-based-halfspace}
    Given a halfspace normal $h$ and a risk metric $\calr$, a risk-based halfspace is given by $\calh^\calr:= \{p \mid h \cdot p + \Tilde{g} \leq 0\}$, where $\Tilde{g} = g^* - (S_{\calo}(h) + S_{-\cala}(h))$, and $g^*$ is the optimal value of the following optimization problem:
    \begin{subequations}
        \begin{align}
            \min \quad & g \\
            \text{subject to} \quad & \calr(\ell(\rv{p})) \leq \delta. \label{eq:risk_based_optimal_separator_constraint}
        \end{align} \label{eq:risk_based_optimal_separator}
    \end{subequations}
\end{defn}
\vspace{-0.5cm}

It is common to approximate the distribution of $\rv{p}$ or $\ell(\rv{p})$ then pose \eqref{eq:risk_based_optimal_separator_constraint} as a chance constraint on the probability of collision hence bounding the value-at-risk (VaR) (e.g. \cite{lindemann2022safe}).
However, VaR is not a coherent risk metric in the sense of Artzner et al. \cite[Def 2.4]{artzner1999coherent}. 
Meanwhile, $\cvar$ is a coherent risk metric that has been advocated for in robotics \cite{majumdar2020should, hakobyan2023distributionally}. 
Intuitively, $\cvar$ measures the expected cost in the tail of the distribution. Thus, it not only accounts for the \emph{frequency} of undesirable events, but also their \emph{severity}.

Since we only have samples of $\rv{p}$ generated by the prediction module, and its underlying distribution is unknown, we use a data-driven \emph{distributionally robust} $\cvar$ risk metric, i.e. $\drcvar$, in \eqref{eq:risk_based_optimal_separator_constraint}. 
In particular, the samples of $\rv{p}$ define an empirical distribution $\hat{\bbp}$.
But, instead of treating $\hat{\bbp}$ as the true distribution, which can give large sampling errors for small $\numsamp$, we consider a Wasserstein distance-based ambiguity set $\calp$ around $\hat{\bbp}$. 
Formally, $\calp = \bbb_\epsilon(\hat{\bbp}) := \{\bbq \in \calm(\Xi) \mid \dwass(\hat{\bbp},\bbq) \leq \epsilon \}$ is the Wasserstein ball containing all distributions with a Wasserstein distance of at most $\epsilon$ from $\hat{\bbp}$. 
Here, $\calm(\Xi)$ is the set of all finite mean distributions supported on $\Xi$ and $\dwass(\cdot, \cdot)$ is the Wasserstein distance. 
Consider $\bbq_1, \bbq_2 \in \calm(\Xi)$ and a norm $\norm{\cdot}$ (we use the 2-norm), the Wassertein distance is defined by
$\dwass(\bbq_1, \bbq_2):= \int_{\Xi^2} \norm{\xi_1 - \xi_2} \Pi(d\xi_1, d\xi_2)$ 
where $\Pi$ is a joint distribution of $\xi_1$ and $\xi$ with marginals $\bbq_1$ and $\bbq_2$ respectively \cite[Definition 3.1]{mohajerin2018data}. 
Intuitively, the Wasserstein distance represents the minimum transportation cost of transporting mass from one distribution into another.
Accordingly, a $\drcvar$ safe halfspace is defined as follows.
\begin{defn}[DR-CVaR Safe Halfspace] \label{def:dr-cvar-halfspace}
    Consider an empirical distribution $\hat{\bbp}$ supported on samples of a predicted obstacle position, a Wasserstein-based ambiguity set $\calp = \bbb_\epsilon(\hat{\bbp})$ and a halfspace normal $h$. The $\drcvar$ safe halfspace is defined as $\calh^{dr} := \{p \mid h \cdot p + g^* \leq 0\}$ where $g^*$ is the optimal value of the following optimization problem:
    \begin{subequations}
        \begin{align}
            \min \quad & g \\
            \text{subject to} \quad & \drcvar_\alpha^\epsilon(\ell(\rv{p})) \leq \delta \label{eq:dr_cvar_optimal_separator_constraint}%
        \end{align} \label{eq:dr_cvar_optimal_separator}%
    \end{subequations}
    where $\drcvar_\alpha^\epsilon(\ell(\rv{p})) := \sup_{\bbp \in \calp} \cvar_\alpha^\bbp(\ell(\rv{p}))$.
\end{defn}

Here, \eqref{eq:dr_cvar_optimal_separator_constraint} is an infinite dimensional constraint. However, since $\ell$ is affine, we can utilize tools from \cite{mohajerin2018data} to obtain a finite-dimensional convex reformulation. 

\begin{prop}
    The optimization problem \eqref{eq:dr_cvar_optimal_separator} 
    with the support $\Xi:= \{p \mid V p \leq v\}$ for the random variable $\rv{p}$
    admits the finite-dimensional reformulation:
    {\small
    \begin{align}
        \inf_{\tau, \lambda, \eta_i, \gamma_{ik}} \quad & g  \label{eq:wce_cor5.1reform} \\
        \text{subject to} \quad & \lambda \epsilon + \frac{1}{\numsamp} \sum_{i=1}^{\numsamp} \eta_i \leq \delta, \nonumber \\
        & a_k \cdot p^i + b_k \Tilde{g} + c_k \tau + \gamma_{ik} \cdot (v-V^\top p^i) \leq \eta_i \nonumber\\
        & \norm{V^\top \gamma_{ik} - a_k}_* \leq \lambda, \quad \gamma_{ik} \geq 0 \nonumber
    \end{align}}
    where $\norm{\cdot}_*$ is the dual norm.
\end{prop}
\begin{proof}
    Using the $\cvar$ definition with \eqref{eq:loss_fxn_with_support} we have:
    \begin{subequations}
        \begin{align}
            & \cvar_\alpha^\bbp(\ell(\rv{p})) \\
            & = \inf_\tau \bbe^{\bbp} \left[ \max\left(-\frac{h \cdot \rv{p} + \Tilde{g} }{\alpha} + (1-\frac{1}{\alpha}) \tau, \tau\right) \right] \\
            & = \inf_\tau \bbe^{\bbp} \left[ \underbrace{\max_k\left(a_k \rv{p} + b_k \Tilde{g} + c_k \tau \right)}_{\bbm:=} \right]
        \end{align}
    \end{subequations}
    where $k \in \{1, 2\}$ and $a_1 = \frac{-h}{\alpha}$, $b_1 = \frac{-1}{\alpha}$, $c_1 = 1 - \frac{1}{\alpha}$, $a_2 = \bbzero_d$, $b_2 = 0$, $c_2 = 1$.
    Then, the $\drcvar$ constraint \eqref{eq:dr_cvar_optimal_separator_constraint} becomes
    \begin{align}
        \drcvar_\alpha^\epsilon(\ell(\rv{p})) \leq \delta &\iff \sup_{\bbp \in \calp} \cvar_\alpha^\bbp(\ell(\rv{p})) \leq \delta \nonumber \\
        \iff \sup_{\bbp \in \calp} \inf_\tau \bbe^{\bbp}[\bbm] \leq \delta  &\Leftarrow \inf_\tau \underbrace{\sup_{\bbp \in \calp} \bbe^{\bbp}[\bbm]}_{\wce:=} \leq \delta
    \end{align}
    where the last line follows by the minimax inequality $\sup \inf (\cdot) \leq \inf \sup (\cdot)$.
    The worst-case expectation ($\wce$) term matches that from \cite[(10)]{mohajerin2018data} with piecewise affine loss functions in the random variable $\rv{p}$ and hence applying \cite[Corollary 5.1-(i)]{mohajerin2018data} results in \eqref{eq:wce_cor5.1reform}.
\end{proof}

\begin{example} \label{ex:cvar_and_drcvar}
    Consider an ego reference position $y^r = [-0.9, -0.8]^\top$ and a nominal obstacle at $p=~[0.5, 0]^\top$ both of radius $r=0.3$. 
    Figure \ref{fig:safe_hs_comp} illustrates safe halfspaces using the expected value (mean), $\cvar$ and $\drcvar$ risk metrics with $h=(p-y^r)/\norm{p-y^r}$ as depicted by the arrow.
    The halfspaces use 100 samples of the obstacle position sampled from the Gaussian random vector $\caln(p, diag(0.01, 0.01))$ where $diag(\cdot)$ is the diagonal matrix.
    We use $\alpha=0.2$, $\delta=0.1$, and $\epsilon \in \{0.05,\ 0.1,\ 0.2\}$. 
    As $\epsilon$ increases, the Wasserstein ball becomes larger, including more distributions, and hence the $\drcvar$ halfspaces become more conservative. When $\epsilon \rightarrow 0$, the $\drcvar$ safe halfspace converges to the $\cvar$ case.
    Here, $y^r$ is safe with respect to all halfspaces except the $\drcvar$ with $\epsilon=0.2$.
    Note that increasing $\delta$ relaxes the safety constraint and would make the halfspaces less conservative.
\end{example}

\begin{figure}
    \centering    
    \includegraphics[trim={1cm 1cm 1cm 1.2cm}, clip, width=.75\linewidth]{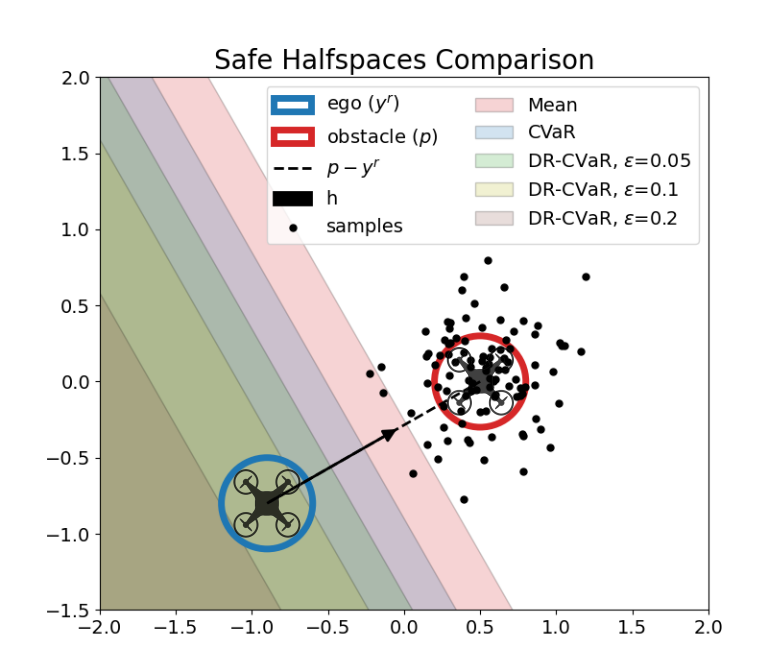}
    \caption{Comparison between safe halfspaces based on the mean, $\cvar$, and $\drcvar$ with different $\epsilon$ values.}
    \label{fig:safe_hs_comp}
    \vspace{-0.1cm}
\end{figure}

\subsection{DR-CVaR Safe Halfspace Guarantees}
If the position of the ego robot is constrained to the DR-CVaR safe halfspaces, we obtain a safety guarantee that bounds collision risk, however, we require a technical assumption on the lightness of the tails\footnote{This assumption hold trivially if $\Xi$ is compact. Since the distribution in our case represents the position of the obstacle, then $\Xi$ is compact as obstacles can always be limited to a finite detection range.} 
for the underlying true probability distribution $\bbp_{true}$ of predicted obstacle positions.
Thus, we have the following lemma.
\begin{lemma}[Concentration Inequality {\cite[Theorem 2]{fournier2015rate}}]
    For a light-tailed distribution $\bbp_{true}$ with
    $\chi~:=~\bbe^{\bbp_{true}}[\text{exp}(\|\rv{p}\|^\rho)] < \infty$ for $\rho > 1$ \cite[Assumption 3.3]{mohajerin2018data} or $\rho > 0$ per \cite{fournier2015rate}, we have
    \begin{align}
        \bbp(\dwass(\bbp_{true}, \hat{\bbp}) \geq \epsilon) \leq \beta
    \end{align}    
    where $\beta$ is a constant term that depends on $\chi, \rho, \numsamp$, and $\epsilon$.    
    Alternatively, $\bbp(\bbp_{true} \in \bbb_\epsilon(\hat{\bbp})) \geq 1-\beta$ for a carefully chosen value of $\numsamp$.
\end{lemma}

Therefore, for a desired concentration bound, it is possible to compute a minimum value of $\numsamp$ to guarantee that satisfying the $\drcvar$ constraint \eqref{eq:dr_cvar_optimal_separator_constraint} ensures the satisfaction of the $\cvar$ variant of it for the true distribution with probability $1-\beta$.
However, this theoretical guarantee is usually quite conservative and generally requires $\numsamp$ to be large making \eqref{eq:wce_cor5.1reform} computationally expensive. Instead, we advocate for treating $\epsilon, \numsamp,$ and $\delta$ as tuneable parameters to achieve a desired safety level that can be validated experimentally.

\section{MPC-Based Safety Filter with DR-CVaR Safe Halfspaces}
We now return to the safe motion planning problem described in Problem \ref{pb:prob_statement}.
Our solution has two steps:
\begin{enumerate}
    \item Computing the safe halfspaces: 
    We interpret the safety constraint $\calr(\calt^r) \leq \delta$ from Definition \ref{def:safety} point-wise in time and per obstacle.
    Thus, the obstacle trajectory samples are used to solve \eqref{eq:dr_cvar_optimal_separator} per obstacle per time step over the horizon $\horizon$ resulting in $\calh^{dr}_i(t+t') \ \forall i \in [1:\Nob],\ t' \in [1:\horizon]$.
    \item MPC Filter: The computed halfspaces are used as constraints in an MPC optimization problem that computes a minimally deviating trajectory from the reference trajectory. The MPC optimization problem is formalized below.
\end{enumerate}

\begin{prob}[MPC-Based Safety Filter] \label{prob:mpc_filter}
Given an ego reference trajectory $\calt^r(t_0)$ at time $t_0$ and linear ego vehicle dynamics \eqref{eq:ego_full_dyn}, find the filtered trajectory $\calt(t_0) = \{x(t_0), \dots, x(t_0 + \horizon) \}$ for the ego vehicle that satisfies the $\drcvar$ safe halfspaces $\calh^{dr}_i(t), \ t \in [t_0+1:t_0+T]$, by solving the finite-horizon optimization problem \eqref{eq:filter_pb} with the objective function \eqref{eq:filter_obj}.
{\small
\begin{align} \label{eq:filter_obj}%
    &J(x(t_0:t_0+\horizon), u(t_0:t_0+\horizon-1)) = \\ 
    & \qquad \sum_{t=t_0}^{\horizon + t_0 - 1} 
        u(t)^\top R(t) u(t) \nonumber \\ 
    & \qquad + \sum_{t=t_0+1}^{\horizon + t_0} 
        (x(t) - x^r(t))^\top Q(t) ((x(t) - x^r(t)) \nonumber
\end{align}}
{\small 
\begin{subequations}
\begin{align} 
    \min_{x, u} \quad & J(x(t_0:t_0+\horizon), u(t_0:t_0+\horizon-1)) & \\
    \text{subject to} \quad & \eqref{eq:ego_dyn},\ u(t) \in \calu(t) \quad \forall t \in [t_0:t_0+\horizon-1] & \\
    & \eqref{eq:ego_observe}, \  x(t_0) = x^r(t_0) \\
    & y(t) \in \caly(t) \cap \left( \bigcap_{i=1}^{\Nob} \calh^{dr}_i(t) \right) & \label{eq:filter_safety_constr}%
\end{align}\label{eq:filter_pb}%
\end{subequations}
}with $\forall t \in [t_0+1:t_0+\horizon]$.
Here, $\caly(t) \subseteq \bbr^{d}$ is a convex position constraint for the ego robot (e.g. environment bounds), $\calu(t) \subseteq \bbr^{m}$ is the convex control input constraint set, and 
$R \in \bbr^{n \times n}$, $Q \in \bbr^{m \times m}$ are symmetric, positive semidefinite cost matrices.
\end{prob}

The MPC safety filter \ref{eq:filter_pb} is a quadratic program (QP) and can be modeled with tools such as CVXPY \cite{diamond2016cvxpy} and solved with many solvers, such as ECOS \cite{domahidi2013ecos}.

\begin{rem}
    We assume that \eqref{eq:filter_pb} is always feasible at $t=0$. 
    The conjunction in \eqref{eq:filter_safety_constr} may become empty or unreachable due to $\calu(t)$ rendering the problem infeasible. 
    In this work, we use the most recent optimal control $u^*(t)$ from solving \eqref{eq:filter_pb} and proceed with the next available control, $u^*(t+1)$, until the problem is solved again or we run out ($u^*(t+T-1)$). 
    Future works will address alternative infeasibility handling approaches.
\end{rem}

\section{Numerical Validation}
We analyze the $\drcvar$ safe halfspace computation time and perform numerical simulations.
All experiments use CVXPY with ECOS and were executed on a Dell Precision 7520 computer with an Intel Xeon E3-1535M v6 CPU and 32GB RAM. 
Experiment code can be found at: \url{https://github.com/TSummersLab/dr-cvar-safety_filtering}. 

\begin{figure}
    \centering
    \includegraphics[width=0.9\linewidth]{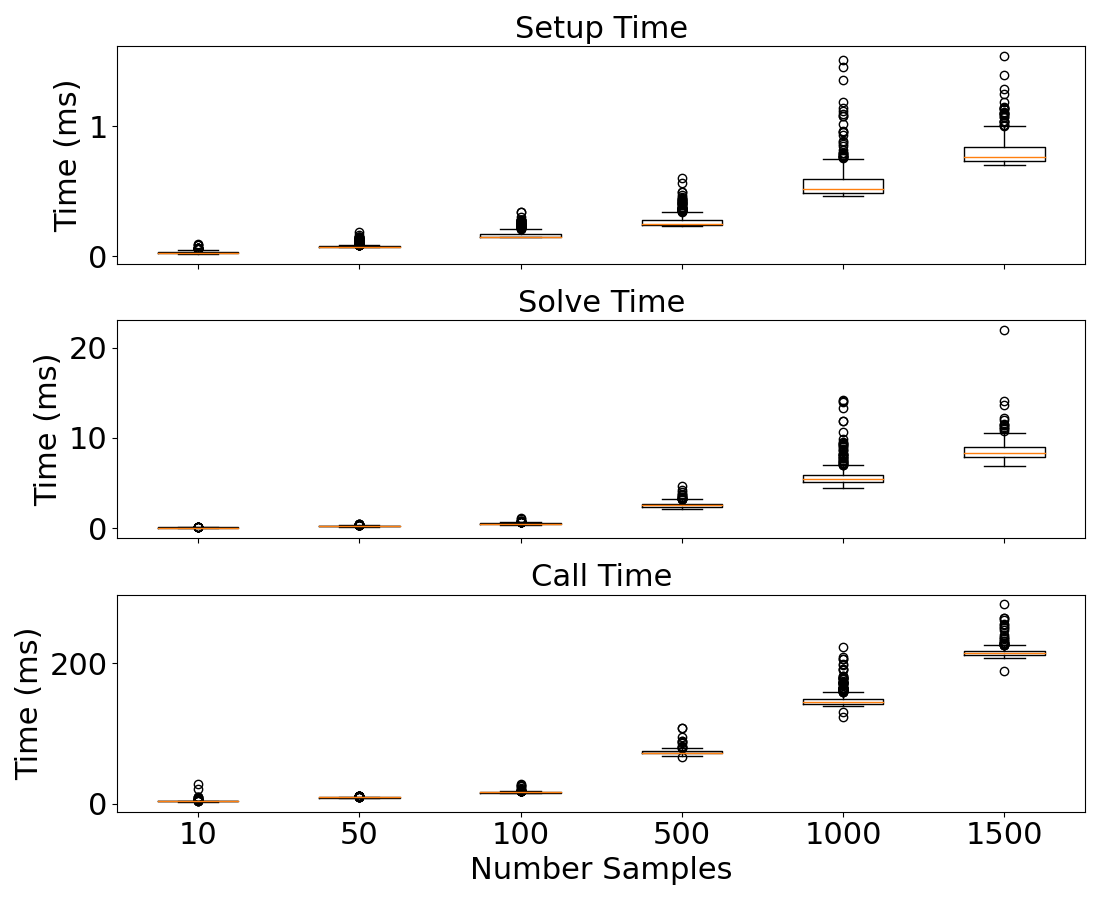}
    \caption{$\drcvar$ safe halfspace computation time}
    \label{fig:compute_times_drcvar}
    \vspace{-0.1cm}
\end{figure}

\begin{table}
\centering
\begin{tabular}{|c||c|c||c|c||c|c|}
\hline
\multirow{2}{*}{Method} & \multicolumn{2}{Sc||}{$\numsamp=50$} & \multicolumn{2}{Sc||}{$\numsamp=500$} & \multicolumn{2}{Sc|}{$\numsamp=1500$} \\
\cline{2-7}
 & {Solve} & {Call} & {Solve} & {Call} & {Solve} & {Call} \\
\hline
$\cvar$ & 0.144 & 2.23 & 1.91 & 4.32 & 5.98 & 8.95 \\
$\drcvar$ & 0.236 & 9.10 & 2.58 & 73.4 & 8.53 & 216 \\
\hline
\end{tabular}
\caption{$\cvar$ and $\drcvar$ safe halfspace average Solve and Call times (ms).}
\label{tab:cvar_dr_cvar_times}
\vspace{-0.1cm}
\end{table}

\subsection{Computation Cost Analysis}
We perform a numerical analysis for the $\drcvar$ safe halfspace computation time.
We find the safe halfspaces as done in Example \ref{ex:cvar_and_drcvar} for various number of samples $\numsamp$. 
For each $\numsamp$, \eqref{eq:dr_cvar_optimal_separator} is solved 500 times with different random samples. We report three times: 
1) the CVXPY reported setup time (Setup Time), 
2) the CVXPY reported solver solve time (Solve Time), and 
3) the time for executing the CVXPY solve method (Call Time).
The results, in milliseconds, are reported in Figure \ref{fig:compute_times_drcvar}.
Since \eqref{eq:dr_cvar_optimal_separator} is a linear program (LP), it can be solved efficiently within a few milliseconds even for a few hundred samples. 
The Call Time includes a large overhead not captured by the Setup Time. 
Additionally, we compare the mean times for finding the $\drcvar$ and $\cvar$ safe halfspaces in Table \ref{tab:cvar_dr_cvar_times}. This reveals that while the Solve Time is only about 50\% higher for the $\drcvar$ problem, the Call Time grows quickly compared to $\cvar$. 
We conclude that the $\drcvar$ safe halfspace formulation is suitable for real-time operation especially if a solver is used directly.

\subsection{Motion Planning Safety Filtering Simulations}
\subsubsection{Simulation Setup} We model the ego and obstacle geometries as circles of radii $r_\cala~=~r_\calo~=~0.3$. 
The ego robot uses double integrator dynamics with 
{\small
$A =~\begin{bmatrix}
    I_2 & I_2 T_s \\ 
    \bbzero_{2,2} & I_2      
\end{bmatrix}, \
B =~\begin{bmatrix}
    \frac{1}{2} I_2 T_s^2 \\
    I_2 T_s
\end{bmatrix},$}
and 
{\small $C =~\begin{bmatrix}
    I_2 & \bbzero_{2,2} 
\end{bmatrix}$}
where $T_s = 0.2$sec is the discrete time step.
The reference trajectory is generated using an obstacle-agnostic MPC-based motion planner with a target goal state. Its details are not discussed since our safety filter is agnostic to the chosen motion planning algorithm. 
Unbeknown to the ego robot, the obstacles are single integrators with {\small $A=I_2,\ B=I_2 T_s,\ C=I_2$}. Sample trajectories are generated by adding a Gaussian random noise $\caln(\bbzero_2, diag(0.01, 0.01))$ to a nominal trajectory that keeps the vehicle aligned with the x-axis at a desired speed. However, when realizing the true position of the obstacle, a Laplace distribution with the same mean and covariance as the Gaussian is sampled.
We use $\horizon=10$, $\alpha=0.2$, $\delta=0.1$, and $\epsilon=0.05$. 

\subsubsection{Filtering in Different Scenarios}
We use three types of reach-avoid motion planning scenarios:
1) Head-on collision, 2) Overtaking, and 3) Intersection. The left column of Figure \ref{fig:planning_scenarios} overlays the trajectories of one experiment using safe halfspaces based on the mean value, $\cvar$, and $\drcvar$ risk metrics. In all three cases, trajectories using $\drcvar$ safe halfspaces achieve the lowest risk, while those using the expected value-based safe halfspace have the highest risk.

\begin{figure}[]
    \centering
    \begin{subfigure}[b]{\linewidth}
        \centering
        \includegraphics[width=0.5\linewidth]{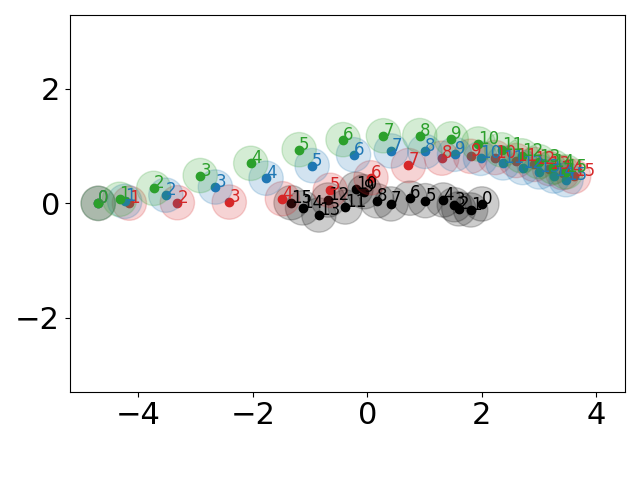}%
        \includegraphics[width=0.5\linewidth]{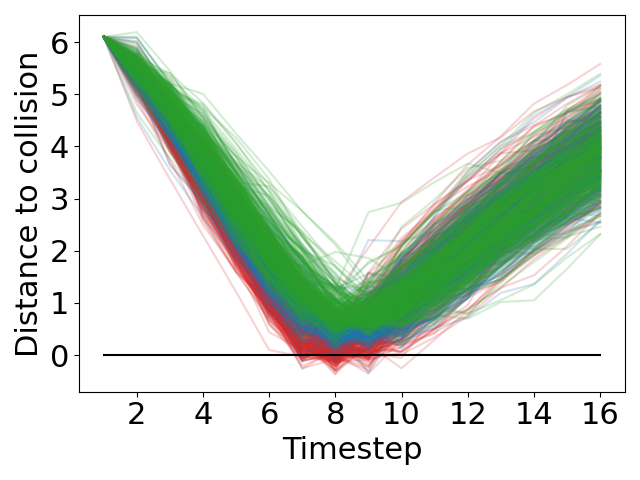}
        \caption{Head-on Collision}
        \label{fig:head_on_all}
    \end{subfigure}
    \vspace{2pt}
    \begin{subfigure}[b]{\linewidth}
        \centering
        \includegraphics[width=0.5\linewidth]{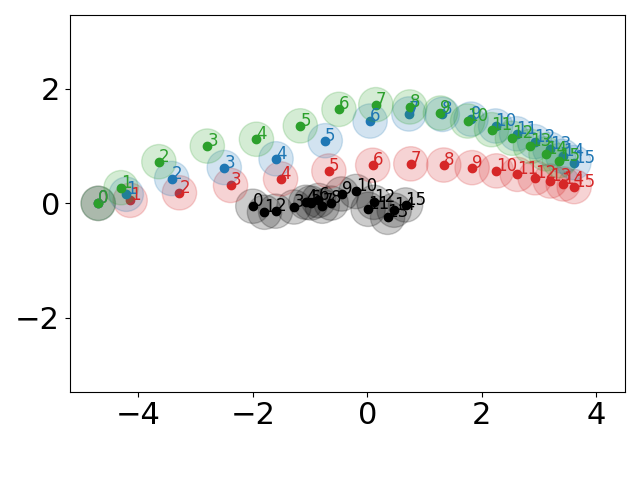}%
        \includegraphics[width=0.5\linewidth]{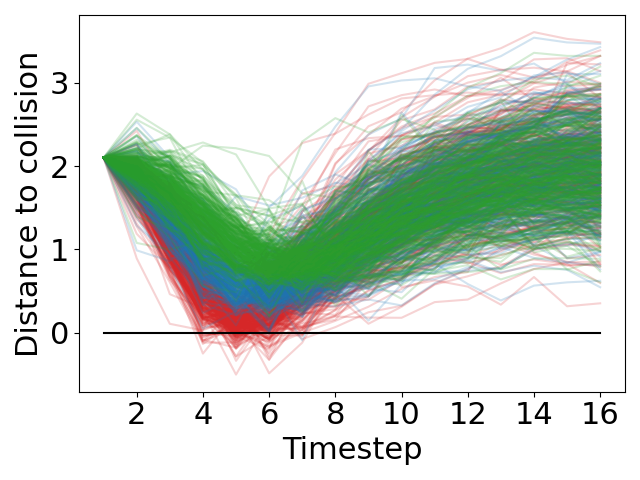}
        \caption{Overtaking}
        \label{fig:overtaking_all}
    \end{subfigure}
    \vspace{2pt}
    \begin{subfigure}[b]{\linewidth}
        \centering
        \includegraphics[width=0.5\linewidth]{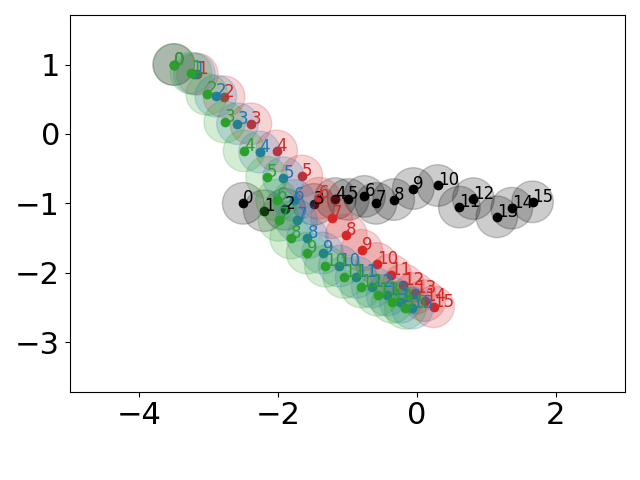}%
        \includegraphics[width=0.5\linewidth]{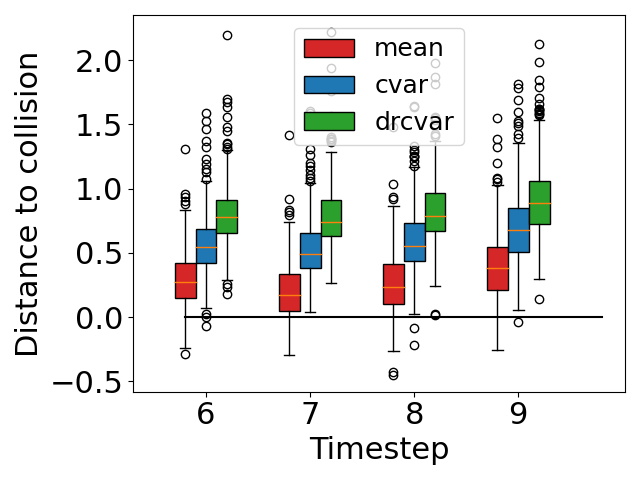}
        \caption{Intersection}
        \label{fig:intersection_all}
    \end{subfigure}
    \caption{Motion Planning Scenarios and their distance to collision statistics across 300 Monte Carlo simulation. Red: Mean safe halfspace $\calh^{\bbe}$. Blue: $\cvar$ safe halfspace $\calh^{cvar}$. Green: $\drcvar$ safe halfspace $\calh^{dr}$ Black: Obstacle.}
    \label{fig:planning_scenarios}
    \vspace{-0.3cm}
\end{figure}

To demonstrate the advantage of using $\drcvar$ halfspaces, we perform a Monte Carlo simulation repeating each experiment 300 times. The distance to collision $\norm{y - p} - r_\cala - r_\calo$ is plotted in the right column of Figures \ref{fig:planning_scenarios} with the last row showing a box plot zoomed-in view.
Clearly, the performance using safe halfspaces based on the mean value result in more frequent collisions in all three cases. 
In the box plot, the worst-case scenario using the mean value safe halfspace is around $-0.45$. With this collision amount, both the ego robot and obstacle would be significantly damaged. 
The worst-case collision using $\cvar$ results in a $-0.22$ distance to collision. Here, it may lead to a less severe collision. 
On the other had, in all three cases, the $\drcvar$ safe halfspace-based formulations avoid collisions with the ego vehicle, avoiding the obstacle even in the worst case scenario.
Thus, our proposed solution can secure the robot's safety and systematically reduce collision risks, even with edge cases in the prediction distribution tail.

\subsubsection{Safety Filtering With Multiple Obstacles}
Consider the motion planning problem in Figure \ref{fig:ego_and_3_obs}, where the ego robot must avoid 3 obstacles. Figure \ref{fig:ego_and_3_obs} plots the overall trajectories as well as the $\drcvar$ safe halfspaces (green polytopes). 
Since the MPC safety filter is a QP, it can be solved efficiently in a few milliseconds. Here, the filter call time takes 7ms on average.
The safety filter becomes infeasible only once, so we fall back to the previously computed set of optimal controls. This is illustrated at time step 5 when the ego robot is not inside any of the green safe polytopes. 
Throughout the experiment, the robot remains sufficiently far away from the obstacles and successfully reaches its goal. 

\begin{figure}
     \centering
     \includegraphics[width=\linewidth, trim={0 0 0 1.3cm}]{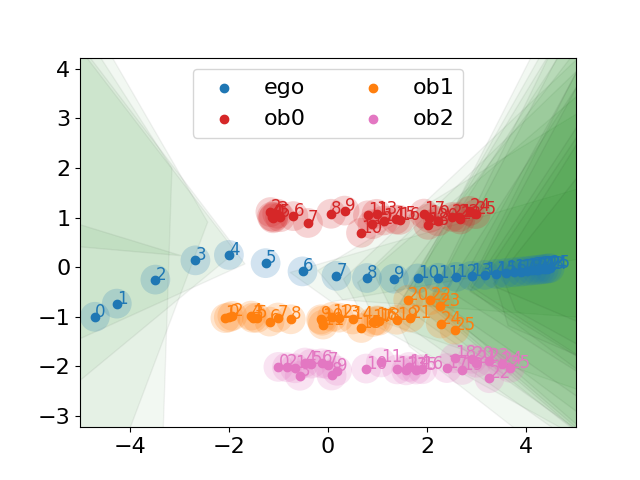}
    \caption{Collision Avoidance with Multiple Obstacles. Green polytopes: $\drcvar$ safe halfspaces. Blue: ego. All other colors: obstacles} \label{fig:ego_and_3_obs}
    \vspace{-0.3cm}
\end{figure}

\section{Conclusion}
In this work we presented a solution that improves a robot's safety when operating in a dynamic environment with prediction uncertainties. 
We posed a DRO problem that computes $\drcvar$ safe halfspaces that bound the $\cvar$ of a signed collision distance under any distribution close the empirical one based on data.
These halfspaces are then used as linear constraints in an MPC safety filter that corrects the ego reference trajectory.
We performed a numerical analysis on the $\drcvar$ safe halfspaces and demonstrated that 1) they can be computed in milliseconds, and 2) they can improve safety in edge cases.
Future directions include a better approach to handling MPC infeasibilities, relaxations of the halfspaces to trade-off safety for performance, and an implementation of the approach on physical hardware. 
\bibliography{refs}

\begin{thebibliography}{10}
\providecommand{\url}[1]{#1}
\csname url@samestyle\endcsname
\providecommand{\newblock}{\relax}
\providecommand{\bibinfo}[2]{#2}
\providecommand{\BIBentrySTDinterwordspacing}{\spaceskip=0pt\relax}
\providecommand{\BIBentryALTinterwordstretchfactor}{4}
\providecommand{\BIBentryALTinterwordspacing}{\spaceskip=\fontdimen2\font plus
\BIBentryALTinterwordstretchfactor\fontdimen3\font minus
  \fontdimen4\font\relax}
\providecommand{\BIBforeignlanguage}[2]{{%
\expandafter\ifx\csname l@#1\endcsname\relax
\typeout{** WARNING: IEEEtran.bst: No hyphenation pattern has been}%
\typeout{** loaded for the language `#1'. Using the pattern for}%
\typeout{** the default language instead.}%
\else
\language=\csname l@#1\endcsname
\fi
#2}}
\providecommand{\BIBdecl}{\relax}
\BIBdecl

\bibitem{yurtsever2020survey}
E.~Yurtsever, J.~Lambert, A.~Carballo, and K.~Takeda, ``A survey of autonomous
  driving: Common practices and emerging technologies,'' \emph{IEEE access},
  vol.~8, pp. 58\,443--58\,469, 2020.

\bibitem{wawrla2019applications}
L.~Wawrla, O.~Maghazei, and T.~Netland, ``Applications of drones in warehouse
  operations,'' \emph{Whitepaper. ETH Zurich, D-MTEC}, p. 212, 2019.

\bibitem{jung2017analysis}
S.~Jung and H.~Kim, ``Analysis of amazon prime air uav delivery service,''
  \emph{Journal of Knowledge Information Technology and Systems}, vol.~12,
  no.~2, pp. 253--266, 2017.

\bibitem{das2015devices}
J.~Das, G.~Cross, C.~Qu, A.~Makineni, P.~Tokekar, Y.~Mulgaonkar, and V.~Kumar,
  ``Devices, systems, and methods for automated monitoring enabling precision
  agriculture,'' in \emph{2015 IEEE international conference on automation
  science and engineering (CASE)}.\hskip 1em plus 0.5em minus 0.4em\relax IEEE,
  2015, pp. 462--469.

\bibitem{martinez2017driving}
C.~M. Martinez, M.~Heucke, F.-Y. Wang, B.~Gao, and D.~Cao, ``Driving style
  recognition for intelligent vehicle control and advanced driver assistance: A
  survey,'' \emph{IEEE Transactions on Intelligent Transportation Systems},
  vol.~19, no.~3, pp. 666--676, 2017.

\bibitem{yurtsever2019risky}
E.~Yurtsever, Y.~Liu, J.~Lambert, C.~Miyajima, E.~Takeuchi, K.~Takeda, and
  J.~H. Hansen, ``Risky action recognition in lane change video clips using
  deep spatiotemporal networks with segmentation mask transfer,'' in \emph{2019
  IEEE Intelligent Transportation Systems Conference (ITSC)}.\hskip 1em plus
  0.5em minus 0.4em\relax IEEE, 2019, pp. 3100--3107.

\bibitem{davies2016google}
A.~Davies, ``Google’s self-driving car caused its first crash,'' in
  \emph{Wired}, 2016.

\bibitem{geng2017scenario}
X.~Geng, H.~Liang, B.~Yu, P.~Zhao, L.~He, and R.~Huang, ``A scenario-adaptive
  driving behavior prediction approach to urban autonomous driving,''
  \emph{Applied Sciences}, vol.~7, no.~4, p. 426, 2017.

\bibitem{kumar2013learning}
P.~Kumar, M.~Perrollaz, S.~Lefevre, and C.~Laugier, ``Learning-based approach
  for online lane change intention prediction,'' in \emph{2013 IEEE Intelligent
  Vehicles Symposium (IV)}.\hskip 1em plus 0.5em minus 0.4em\relax IEEE, 2013,
  pp. 797--802.

\bibitem{philion2020lift}
J.~Philion and S.~Fidler, ``Lift, splat, shoot: Encoding images from arbitrary
  camera rigs by implicitly unprojecting to 3d,'' in \emph{Computer
  Vision--ECCV 2020: 16th European Conference, Glasgow, UK, August 23--28,
  2020, Proceedings, Part XIV 16}.\hskip 1em plus 0.5em minus 0.4em\relax
  Springer, 2020, pp. 194--210.

\bibitem{hu2021fiery}
A.~Hu, Z.~Murez, N.~Mohan, S.~Dudas, J.~Hawke, V.~Badrinarayanan, R.~Cipolla,
  and A.~Kendall, ``Fiery: Future instance prediction in bird's-eye view from
  surround monocular cameras,'' in \emph{Proceedings of the IEEE/CVF
  International Conference on Computer Vision}, 2021, pp. 15\,273--15\,282.

\bibitem{zhou2022long}
W.~Zhou, Z.~Cao, Y.~Xu, N.~Deng, X.~Liu, K.~Jiang, and D.~Yang, ``Long-tail
  prediction uncertainty aware trajectory planning for self-driving vehicles,''
  in \emph{2022 IEEE 25th International Conference on Intelligent
  Transportation Systems (ITSC)}.\hskip 1em plus 0.5em minus 0.4em\relax IEEE,
  2022, pp. 1275--1282.

\bibitem{hakobyan2021wasserstein}
A.~Hakobyan and I.~Yang, ``Wasserstein distributionally robust motion control
  for collision avoidance using conditional value-at-risk,'' \emph{IEEE
  Transactions on Robotics}, vol.~38, no.~2, pp. 939--957, 2021.

\bibitem{lindemann2022safe}
L.~Lindemann, M.~Cleaveland, G.~Shim, and G.~J. Pappas, ``Safe planning in
  dynamic environments using conformal prediction,'' \emph{arXiv preprint
  arXiv:2210.10254}, 2022.

\bibitem{tearle2021predictive}
B.~Tearle, K.~P. Wabersich, A.~Carron, and M.~N. Zeilinger, ``A predictive
  safety filter for learning-based racing control,'' \emph{IEEE Robotics and
  Automation Letters}, vol.~6, no.~4, pp. 7635--7642, 2021.

\bibitem{vinod2022safe}
A.~P. Vinod, S.~Safaoui, A.~Chakrabarty, R.~Quirynen, N.~Yoshikawa, and
  S.~Di~Cairano, ``Safe multi-agent motion planning via filtered reinforcement
  learning,'' in \emph{2022 International Conference on Robotics and Automation
  (ICRA)}.\hskip 1em plus 0.5em minus 0.4em\relax IEEE, 2022, pp. 7270--7276.

\bibitem{phan2022driving}
T.~Phan-Minh, F.~Howington, T.-S. Chu, S.~U. Lee, M.~S. Tomov, N.~Li, C.~Dicle,
  S.~Findler, F.~Suarez-Ruiz, R.~Beaudoin \emph{et~al.}, ``Driving in real life
  with inverse reinforcement learning,'' \emph{arXiv preprint
  arXiv:2206.03004}, 2022.

\bibitem{salzmann2020trajectron++}
T.~Salzmann, B.~Ivanovic, P.~Chakravarty, and M.~Pavone, ``Trajectron++:
  Dynamically-feasible trajectory forecasting with heterogeneous data,'' in
  \emph{Computer Vision--ECCV 2020: 16th European Conference, Glasgow, UK,
  August 23--28, 2020, Proceedings, Part XVIII 16}.\hskip 1em plus 0.5em minus
  0.4em\relax Springer, 2020, pp. 683--700.

\bibitem{boyd2004convex}
S.~Boyd, S.~P. Boyd, and L.~Vandenberghe, \emph{Convex optimization}.\hskip 1em
  plus 0.5em minus 0.4em\relax Cambridge university press, 2004.

\bibitem{artzner1999coherent}
P.~Artzner, F.~Delbaen, J.-M. Eber, and D.~Heath, ``Coherent measures of
  risk,'' \emph{Mathematical finance}, vol.~9, no.~3, pp. 203--228, 1999.

\bibitem{majumdar2020should}
A.~Majumdar and M.~Pavone, ``How should a robot assess risk? towards an
  axiomatic theory of risk in robotics,'' in \emph{Robotics Research: The 18th
  International Symposium ISRR}.\hskip 1em plus 0.5em minus 0.4em\relax
  Springer, 2020, pp. 75--84.

\bibitem{hakobyan2023distributionally}
A.~Hakobyan and I.~Yang, ``Distributionally robust optimization with unscented
  transform for learning-based motion control in dynamic environments,'' in
  \emph{2023 IEEE International Conference on Robotics and Automation
  (ICRA)}.\hskip 1em plus 0.5em minus 0.4em\relax IEEE, 2023, pp. 3225--3232.

\bibitem{mohajerin2018data}
P.~Mohajerin~Esfahani and D.~Kuhn, ``Data-driven distributionally robust
  optimization using the wasserstein metric: Performance guarantees and
  tractable reformulations,'' \emph{Mathematical Programming}, vol. 171, no.~1,
  pp. 115--166, 2018.

\bibitem{fournier2015rate}
N.~Fournier and A.~Guillin, ``On the rate of convergence in wasserstein
  distance of the empirical measure,'' \emph{Probability theory and related
  fields}, vol. 162, no. 3-4, pp. 707--738, 2015.

\bibitem{diamond2016cvxpy}
S.~Diamond and S.~Boyd, ``Cvxpy: A python-embedded modeling language for convex
  optimization,'' \emph{The Journal of Machine Learning Research}, vol.~17,
  no.~1, pp. 2909--2913, 2016.

\bibitem{domahidi2013ecos}
A.~Domahidi, E.~Chu, and S.~Boyd, ``Ecos: An socp solver for embedded
  systems,'' in \emph{2013 European control conference (ECC)}.\hskip 1em plus
  0.5em minus 0.4em\relax IEEE, 2013, pp. 3071--3076.

\end{thebibliography}
\end{document}